\documentclass{article}

\PassOptionsToPackage{numbers, compress}{natbib}
     \usepackage[final]{neurips_2019}

\usepackage[utf8]{inputenc} % allow utf-8 input
\usepackage[T1]{fontenc}    % use 8-bit T1 fonts
\usepackage{hyperref}       % hyperlinks
\usepackage{url}            % simple URL typesetting
\usepackage{booktabs}       % professional-quality tables
\usepackage{amsfonts}       % blackboard math symbols
\usepackage{nicefrac}       % compact symbols for 1/2, etc.
\usepackage{microtype}      % microtypography
\usepackage{graphicx} 
\usepackage{amssymb}
\usepackage{amsthm}
\usepackage{amsmath} 
\usepackage{algorithmic}
\usepackage{algorithm}
\usepackage{color}
\usepackage{natbib}
\usepackage{amsthm}

\newtheorem{theorem}[]{Theorem}
\newtheorem{remark}[]{Remark}

\newcommand{\bo}{\mathbf}
\newcommand{\cl}{\mathcal}

\title{Stochastic Shared Embeddings: Data-driven Regularization of Embedding Layers}

\author{%
  Liwei Wu \\
  Department of Statistics\\
  University of California, Davis\\
  Davis, CA 95616 \\
  \texttt{liwu@ucdavis.edu} \\
  \And
  Shuqing Li \\
  Department of Computer Science\\
  University of California, Davis\\
  Davis, CA 95616 \\
  \texttt{qshli@ucdavis.edu}
  \And
  Cho-Jui Hsieh \\ 
  Department of Computer Science\\
  University of California, Los Angles\\
  Los Angles, CA 90095 \\
  \texttt{chohsieh@cs.ucla.edu}
  \And
  James Sharpnack \\
  Department of Statistics\\
  University of California, Davis\\
  Davis, CA 95616 \\
  \texttt{jsharpna@ucdavis.edu} \\
}

\begin{document}

\maketitle

\begin{abstract}
In deep neural nets, lower level embedding layers account for a large portion of the total number of parameters. Tikhonov regularization, graph-based regularization, and hard parameter sharing are approaches that introduce explicit biases into training in a hope to reduce statistical complexity. Alternatively, we propose stochastic shared embeddings (SSE), a data-driven approach to regularizing embedding layers, which stochastically transitions between embeddings during stochastic gradient descent (SGD). Because SSE integrates seamlessly with existing SGD algorithms, it can be used with only minor modifications when training large scale neural networks. We develop two versions of SSE: SSE-Graph using knowledge graphs of embeddings; SSE-SE using no prior information. We provide theoretical guarantees for our method and show its empirical effectiveness on 6 distinct tasks, from simple neural networks with one hidden layer in recommender systems, to the transformer and BERT in natural languages. We find that when used along with widely-used regularization methods such as weight decay and dropout, our proposed SSE can further reduce overfitting, which often leads to more favorable generalization results.    
\end{abstract}

\section{Introduction}
Recently, embedding representations have been widely used in almost all AI-related fields, from feature maps \cite{krizhevsky2012imagenet} in computer vision, to word embeddings \cite{mikolov2013distributed, pennington2014glove}  in natural language processing, to user/item embeddings \cite{mnih2008probabilistic, hu2008collaborative} in recommender systems. Usually, the embeddings are high-dimensional vectors. Take language models for example, in GPT \cite{radford2018improving} and Bert-Base model \cite{devlin2018bert}, 768-dimensional vectors are used to represent words. Bert-Large model utilizes 1024-dimensional vectors and GPT-2 \cite{radford2019language} may have used even higher dimensions in their unreleased large models. In recommender systems, things are slightly different: the dimension of user/item embeddings are usually set to be reasonably small, 50 or 100, but the number of users and items is on a much bigger scale. Contrast this with the fact that the size of word vocabulary that normally ranges from 50,000 to 150,000, the number of users and items can be millions or even billions in large-scale real-world commercial recommender systems \cite{bennett2007netflix}. 

Given the massive number of parameters in modern neural networks with embedding layers, mitigating over-parameterization can play an important role in preventing over-fitting in deep learning. 
We propose a regularization method, Stochastic Shared Embeddings (SSE), that uses prior information about similarities between embeddings, such as semantically and grammatically related words in natural languages or real-world users who share social relationships. Critically, SSE progresses by stochastically transitioning between embeddings as opposed to a more brute-force regularization such as graph-based Laplacian regularization and ridge regularization.
Thus, SSE integrates seamlessly with existing stochastic optimization methods and the resulting regularization is data-driven.
%To achieve the goal, we think we first need to understand the sharing mechanisms of embeddings for those semantically and grammatically related words in natural languages or real-world users who share complicated social relationships beyond the Internet.

% This paper is intended as a first try from the angle of stochastic sharing mechanisms and is ideologically very different from non-stochastic $l_2$ regularization (i.e. weight decay) \cite{krogh1992simple} and hard/soft parameter sharing \cite{nowlan1992simplifying}. We do not claim that our method to be the optimal solution to this problem, and to be honest, we do not think we have figured out what the ideal sharing mechanisms should be for embeddings. But we do think our proposed stochastically sharing among highly related embeddings is a reasonable solution, which makes a lot of senses and is strongly supported by empirical results. Moreover, it is very simple to use without the need to significantly modify the existing code bases.  

We will begin the paper with the mathematical formulation of the problem, propose SSE, and provide the motivations behind SSE. 
We provide a theoretical analysis of SSE that can be compared with excess risk bounds based on empirical Rademacher complexity.
%We see that our analogous notions of complexity for SSE may be significantly smaller than the Rademacher complexity because the transition probabilities have the effect of `smoothing' the complexity.
We then conducted experiments for a total of 6 tasks from simple neural networks with one hidden layer in recommender systems, to the transformer and BERT in natural languages and find that when used along with widely-used regularization methods such as weight decay and dropout, our proposed methods can further reduce over-fitting, which often leads to more favorable generalization results.  

\section{Related Work}
Regularization techniques are used  to control model complexity and avoid over-fitting. 
$\ell_2$ regularization~\cite{hoerl1970ridge} is the most widely used approach and has been used in many matrix factorization models in recommender systems; $\ell_1$ regularization~\cite{tibshirani1996regression} is used when a sparse model is preferred. For deep neural networks, it has been shown that $\ell_p$  regularizations are often too weak, while dropout~\cite{hinton2012improving,srivastava2014dropout} is more effective in practice.  There are many other regularization techniques, including parameter sharing \cite{goodfellow2016deep}, max-norm regularization \cite{srebro2005maximum}, gradient clipping \cite{pascanu2013difficulty}, etc.

Our proposed SSE-graph is very different from graph Laplacian regularization \cite{cai2011graph}, in which the distances of any two embeddings connected over the graph are directly penalized.
%, but there is no direct relationship between this and the final loss. 
Hard parameter sharing uses one embedding to replace all distinct embeddings in the same group, which inevitably introduces a significant bias. 
Soft parameter sharing \cite{nowlan1992simplifying} is similar to the graph Laplacian, penalizing the $l_2$ distances between any two embeddings. 
These methods have no dependence on the loss, while the proposed SSE-graph method is data-driven in that the loss influences the effect of regularization.
Unlike graph Laplacian regularization, hard and soft parameter sharing, our method is stochastic by nature.
This allows our model to enjoy similar advantages as dropout \cite{srivastava2014dropout}.

% One way to understand our SSE algorithm is that we are training exponentially many such neural networks at the same time, in which each network has an additional remapping layer. The function of such a layer is to map embeddings from the original embedding table to a new embedding table in a predefined and deterministic fashion. 
% However, it is obvious that directly optimizing the sum of losses of exponentially many networks is highly inefficient and computationally unfeasible. Motivated by the dropout paper, we can take advantage of stochasticity and solve the original neural network with a slight modification on the backpropagation procedure.

Interestingly, in the original BERT model's pre-training stage \cite{devlin2018bert}, a variant of SSE-SE is already implicitly used for token embeddings but for a different reason. 
In \cite{devlin2018bert}, the authors masked 15\% of words and 10\% of the time replaced the [mask] token with a random token. 
In the next section, we discuss how SSE-SE differs from this heuristic.
% That is roughly equivalent to a SSE probability of 0.015 for replacing input word-piece embeddings.
% We do not think the BERT authors explained the benefits accurately here.
% The authors stated their motivation of doing so is to mitigate a mismatch between pre-training and fine-tuning, since [mask] token is never seen in the fine-tuning stage. But we will argue the gain is due to the fact that replacing [mask] token with a random token has the same effects of replacing the tokens' embeddings as our SSE-SE does and this would allow us to mitigate over-parameterization in the embedding layer and to learn better embedding representations without the need to reduce the number of parameters of the model.
Another closely related technique to ours is the label smoothing \cite{szegedy2016rethinking}, which is widely used in the computer vision community.
We find that in the classification setting if we apply SSE-SE to one-hot encodings associated with output $y_i$ only, our SSE-SE is closely related to the label smoothing, which can be treated as a special case of our proposed method. 
 
% \section{Motivation}
% {\color{red}Cho: maybe we can combine motivation and formulations and maybe remove the story.}
% The most related works are weight decay, hard/soft parameter sharing and dropout. Both can be used to mitigate over-parameterization in embedding layer of neural nets. However, stochasticity is missing from weight decay, hard and soft parameter sharing, and we find existing regularization techniques are not sufficient to complete resolve over-parameterization in embedding layer. So we improve upon hard/soft parameter sharing and propose a new method called Stochastic Shared Embeddings (\PassOptionsToPackage{}{}SSE) to force embeddings to share common information stochastically.

% See the motivation section of dropout paper and come up with a good story.
% What about the story of Hua Mulan? 

\begin{algorithm}[tb]

  \caption{SSE-Graph for Neural Networks with Embeddings}
  \label{alg:see-graph}

\begin{algorithmic}[1]
  \STATE {\bfseries Input:} input $x_i$, label $y_i$, backpropagate $T$ steps, mini-batch size $m$, knowledge graphs on embeddings $\{E_1, \dots, E_M \}$
  \STATE Define $p_l(., . | \Phi)$ based on knowledge graphs on embeddings, $l = 1,\ldots, M$
  \FOR{$t=1$ {\bfseries to} $T$}
    \STATE Sample one mini-batch $\{x_1, \dots, x_m\}$
     \FOR{$i=1$ {\bfseries to} $m$}
        \STATE Identify the set of embeddings $\mathcal{S}_i = \{E_1[j^i_1], \dots, E_M[j^i_M]\}$ for input $x_i$ and label $y_i$
        \FOR{each embedding $E_l[j^i_l] \in \mathcal{S}_i$}
            \STATE  Replace $E_l[j^i_l]$ with $E_l[k_l]$, where $k_l \sim p_l(j^i_l, . | \Phi)$
        \ENDFOR
     \ENDFOR
    \STATE Forward and backward pass with the new embeddings
  \ENDFOR
  \STATE Return embeddings $\{E_1, \dots, E_M \}$, and neural network parameters $\Theta$
 \end{algorithmic}

\end{algorithm}

% Our method takes motivation from an ancient Chinese tale. In the tale, Hua Mulan, disguised as a man, takes her aged father's place in the army and becomes a legendary Chinese warrior without anyone knowing she is a girl and returns home when the war is won. Hua Mulan replaces her father but can still achieve the same task potentially because they share many characteristics with each other. If Hua Mulan is represented as an embedding and her father is as another embedding, then Hua Mulan temporarily uses her own embedding for the task in which her father's embedding is supposed to work best without sacrificing the quality of accomplishing the same task. It in fact turns out that she achieves the same task even better and becomes a legendary general. Does this mean sometimes replacing the embedding meant for the task with a different but similar embedding would improve the performance of achieving the same original task?

\section{Stochastic Shared Embeddings}
% \begin{itemize}
%     \item View matrix factorization as a simple example of one hidden layer neural network: 2 embedding tables, 2 embeddings each sgd update.
%     \item View BPR as another example of one hidden layer neural network but a different loss, 3 embeddings each sgd upate (maybe omit?).
%     \item View SASRec as 7 layers neural network with 2 transformer blocks, each Adam update involving $m \times T $ embeddings. 
% \end{itemize}

Throughout this paper, the network input $x_i$ and label $y_i$ will be encoded into indices $j_1^i,\ldots, j_M^i$ which are elements of $\mathcal I_1 \times \ldots \mathcal I_M$, the index sets of embedding tables.
A typical choice is that the indices are the encoding of a dictionary for words in natural language applications, or user and item tables in recommendation systems. 
Each index, $j_l$, within the $l$th table, is associated with an embedding $E_l[j_l]$ which is a trainable vector in $\mathbb R^{d_l}$.
The embeddings associated with label $y_i$ are usually non-trainable one-hot vectors corresponding to label look-up tables while embeddings associated with input $x_i$ are trainable embedding vectors for embedding look-up tables. 
In natural language applications, we appropriately modify this framework to accommodate sequences such as sentences.

The loss function can be written as the functions of embeddings:
\begin{equation}\label{eq:obj}
    R_n(\Theta) = \sum_i \ell(x_i, y_i | \Theta) = \sum_i \ell(E_1[j^i_1], \dots, E_M[j_M^i] | \Theta),
\end{equation}
%\vspace{-.1cm}
%{\color{red}(Cho: some problem in notation. $(e_j, \dots, e_k)$ should be depend on $i$ (otherwise $i$ does not appear in the function))}
%where $e_1(j_1^i), \dots, e_E(j_E^i) \in \mathbb{R}^{d_1} \times \ldots \times \mathbb{R}^{d_E}$ are embeddings associated with input $x_i$, 
where $y_i$ is the label and $\Theta$ encompasses all trainable parameters including the embeddings, $\{ E_l[j_l]: j_l \in \mathcal I_l\}$.
%and embedding indices $j_1^i, \ldots, j_E^i$ for the $E$ embedding tables. 
The loss function $\ell$ is a mapping from embedding spaces to the reals.
% and is parametrized by trainable parameters $\Theta$ such as weights and biases in neural networks.
For text input, each $E_l[j^i_l]$ is a word embedding vector in the input sentence or document. For recommender systems, usually there are two embedding look-up tables: one for users and one for items \cite{he2017neural}. So the objective function, such as mean squared loss or some ranking losses, will comprise both user and item embeddings for each input. 
We can more succinctly write the matrix of all embeddings for the $i$th sample as $\bo E[\bo j^i] = (E_1[j_1^i], \ldots, E_M[j_M^i])$ where $\bo j^i = (j^i_1,\ldots, j^i_M) \in \cl I$.
By an abuse of notation we write the loss as a function of the embedding matrix, $\ell(\bo E[\bo j^i] | \Theta)$.

\begin{figure}
  \centering
  \includegraphics[width=0.8\linewidth]{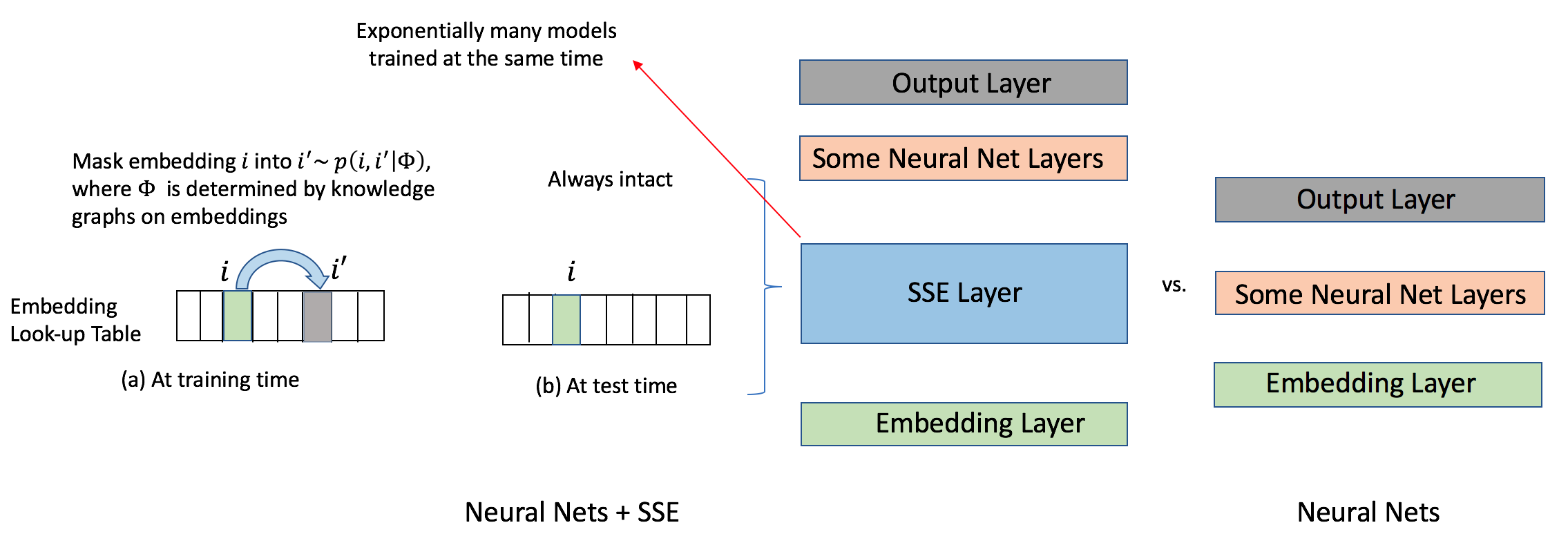}
  \vspace{-10pt}
  \caption{SSE-Graph described in Algorithm~\ref{alg:see-graph} and Figure~\ref{fig:sse-graph} can be viewed as adding exponentially many distinct reordering layers above the embedding layer. 
%   The function of such a reordering layer is to reorder an embedding look-up table defined in embedding layer into a different but sensible embedding table with some slices of embedding table replaced, before feeding into next layer. 
  A modified backpropagation procedure in Algorithm~\ref{alg:see-graph} is used to train exponentially many such neural networks at the same time.
  }
  \vspace{-12pt}
  \label{fig:train_test}
\end{figure}

%Add plot here. Copy SQ plot in SQL-Rank paper but in neural networks architecture like figure 1 in dropout paper. 
%\subsection{Motivations}

Suppose that we have access to knowledge graphs \cite{miller1995wordnet, lehmann2015dbpedia} over embeddings, and we have a prior belief that two embeddings will share information and replacing one with the other should not incur a significant change in the loss distribution. For example, if two movies are both comedies and they are starred by the same actors, it is very likely that for the same user, replacing one comedy movie with the other comedy movie will result in little change in the loss distribution. 
In stochastic optimization, we can replace the loss gradient for one movie's embedding with the other similar movie's embedding, and this will not significantly bias the gradient if the prior belief is accurate.
On the other hand, if this exchange is stochastic, then it will act to smooth the gradient steps in the long run, thus regularizing the gradient updates.

\begin{figure}
  \centering
  \includegraphics[width=0.835\linewidth]{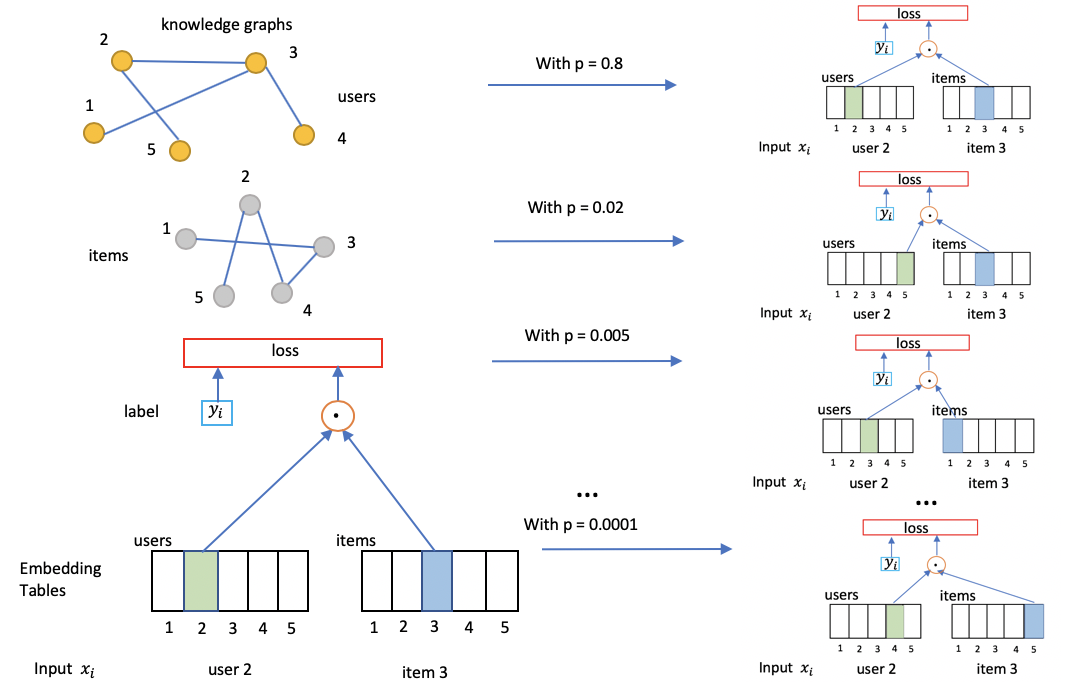}
  \caption{Illustration of how SSE-Graph algorithm in Figure~\ref{fig:train_test} works for a simple neural network.
%   The actual transition probabilities to be used is determined by the knowledge graphs and specific datasets and requires only minimal tuning.
  }\vspace{-12pt}
  \label{fig:sse-graph}
\end{figure}

\subsection{General SSE with Knowledge Graphs: SSE-Graph}
% Need to add plots here, like figure 2 and 3 in dropout paper. 

%Due to the large amount of parameters used in embedding look-up tables, neural networks are easy to overfit when trained on tasks with a large number of input embeddings or large output space. 
%In the following, we will introduce Stochastic Shared Embeddings (SSE), a simple and novel regularization method specifically developed for large embedding tables to improve the generalization performance. The core idea of SSE is to replace each embedding {\it stochastically} 
% {\color{red}(What do you mean by temporarily?} {\color{blue} i guess i mean embeddings are not replaced permanently)} 
%at each iteration with another closely related embedding, forcing common information to be shared stochastically across embeddings. 
%On contrast, our method is stochastic and shares many similarities with dropout \cite{srivastava2014dropout}. The main difference is that we do not omit some hidden units in embeddings but rather we treat one embedding as a whole in the stochastic sharing process. We find stochasticity is key to the success of both methods and using both simultaneously leads to better performances. }{\color{red}(Cho: maybe we can remove this sentence or move to other place? It seems wierd to say how our algorithm is different from others before introducing it. )}
%We name our new method Stochastic Shared Embeddings (SSE), whose core idea is to replace each embedding {\it temporarily and stochastically} with another closely related embedding to force common information to be shared stochastically across embeddings. 

Instead of optimizing objective function $R_n(\Theta)$ in \eqref{eq:obj}, SSE-Graph described in Algorithm~\ref{alg:see-graph}, Figure~\ref{fig:train_test}, and Figure~\ref{fig:sse-graph} is approximately optimizing the objective function below:
\begin{equation}\label{eq:single_obj}
    S_n(\Theta) = \sum_i \sum_{\bo k \in \cl I} p(\bo j^i, \bo k|\Phi) \ell(\bo E[\bo k] | \Theta),
\end{equation}
where $p(\bo j, \bo k | \Phi)$ is the transition probability (with parameters $\Phi$) of exchanging the encoding vector $\bo j \in \mathcal I$ with a new encoding vector $\bo k \in \mathcal I$ in the Cartesian product index set of all embedding tables.
When there is a single embedding table ($M =1$) then there are no hard restrictions on the transition probabilities, $p(.,.)$, but when there are multiple tables ($M > 1$) then we will enforce that $p(.,.)$ takes a tensor product form (see \eqref{eq:multiple_obj}).
When we are assuming that there is only a single embedding table ($M=1$) we will not bold $j, E[j]$ and suppress their indices.
%The form above implies that there is a single embedding table and a single embedding per data point $(x_i, y_i)$, but as we will see this can accommodate multiple embedding tables and multiple embeddings per $(x_i, y_i)$.

In the single embedding table case, $M=1$, there are many ways to define transition probability from $j$ to $k$. One simple and effective way is to use a random walk (with random restart and self-loops) on a knowledge graph $\mathcal{G}$, i.e.~when embedding $j$ is connected with $k$ but not with $l$, we can set the ratio of  $p(j, k|\Phi)$ and  $p(j, l|\Phi)$ to be a constant greater than $1$. In more formal notation, we have  \begin{equation}\label{eq:transition1}
   j \sim k, j \not\sim l \longrightarrow  p(j, k|\Phi)/p(j, l|\Phi) = \rho, 
\end{equation} where $\rho > 1$ and is a tuning parameter. It is motivated by the fact that embeddings connected with each other in knowledge graphs should bear more resemblance and thus be more likely replaced by each other. 
Also, we let 
%\begin{equation}\label{eq:transition2}
    $p(j, j|\Phi) = 1 - p_0,$
%\end{equation} 
where $p_0$ is called the {\it SSE probability} and embedding retainment probability is $1-p_0$. We treat both $p_0$ and $\rho$ as tuning hyper-parameters in experiments. With \eqref{eq:transition1} and $\sum_k p(j, k|\Phi) = 1$, we can derive transition probabilities between any two embeddings to fill out the transition probability table.

When there are multiple embedding tables, $M > 1$, then we will force that the transition from $\bo j$ to $\bo k$ can be thought of as independent transitions from $j_l$ to $k_l$ within embedding table $l$ (and index set $\cl I_l$).
Each table may have its own knowledge graph, resulting in its own transition probabilities $p_l(.,.)$.
The more general form of the SSE-graph objective is given below:

\begin{equation}\label{eq:multiple_obj}
    S_n(\Theta) =  \sum_i \sum_{k_1, \dots, k_M} p_1(j^i_1, k_1 |\Phi)\cdots p_M(j^i_M, k_M |\Phi) \ell(E_1[k_1], \dots, E_M[k_M] | \Theta),
\end{equation}
Intuitively, this SSE objective could reduce the variance of the estimator.
%and can serve as a regularizer to prevent over-fitting. 

Optimizing \eqref{eq:multiple_obj} with SGD or its variants (Adagrad \cite{duchi2011adaptive}, Adam \cite{kingma2014adam}) is simple. We just need to randomly switch each original embedding tensor $\bo E[\bo j^i]$ with another embedding tensor $\bo E[\bo k]$ randomly sampled according to the transition probability (see Algorithm~\ref{alg:see-graph}). This is equivalent to have a randomized embedding look-up layer as shown in Figure~\ref{fig:train_test}. 

We can also accommodate sequences of embeddings, which commonly occur in natural language application, by considering
 $(j^i_{l,1}, k_{l,1}), \dots, (j^i_{l,n^i_l}, k_{l,n^i_l})$ instead of $(j^i_l, k_l)$ for $l$-th embedding table in \eqref{eq:multiple_obj}, where $1 \leq l \leq M$ and $n^i_l$ is the number of embeddings in table $l$ that are associated with $(x_i, y_i)$.
 When there is more than one embedding look-up table, we sometimes prefer to use different $p_0$ and $\rho$ for different look-up tables in \eqref{eq:transition1} and the SSE probability constraint. For example, in recommender systems, we would use $p_u, \rho_u$ for user embedding table and $p_i, \rho_i$ for item embedding table.
% either thinking of the index sets $\mathcal I_l$ as containing sequences, or by generating sequences of embedding tables.

We find that SSE with knowledge graphs, i.e., SSE-Graph, can force similar embeddings to cluster when compared to the original neural network without SSE-Graph. In Figure~\ref{fig:pca}, one can easily see that more embeddings tend to cluster into 2 singularities after applying SSE-Graph when embeddings are projected into 3D spaces using PCA. Interestingly, a similar phenomenon occurs when assuming the knowledge graph is a complete graph, which we would introduce as SSE-SE below.

\begin{figure*}
\centering
\begin{tabular}{ccc}
\hspace{-8pt}\includegraphics[width=0.3\linewidth]{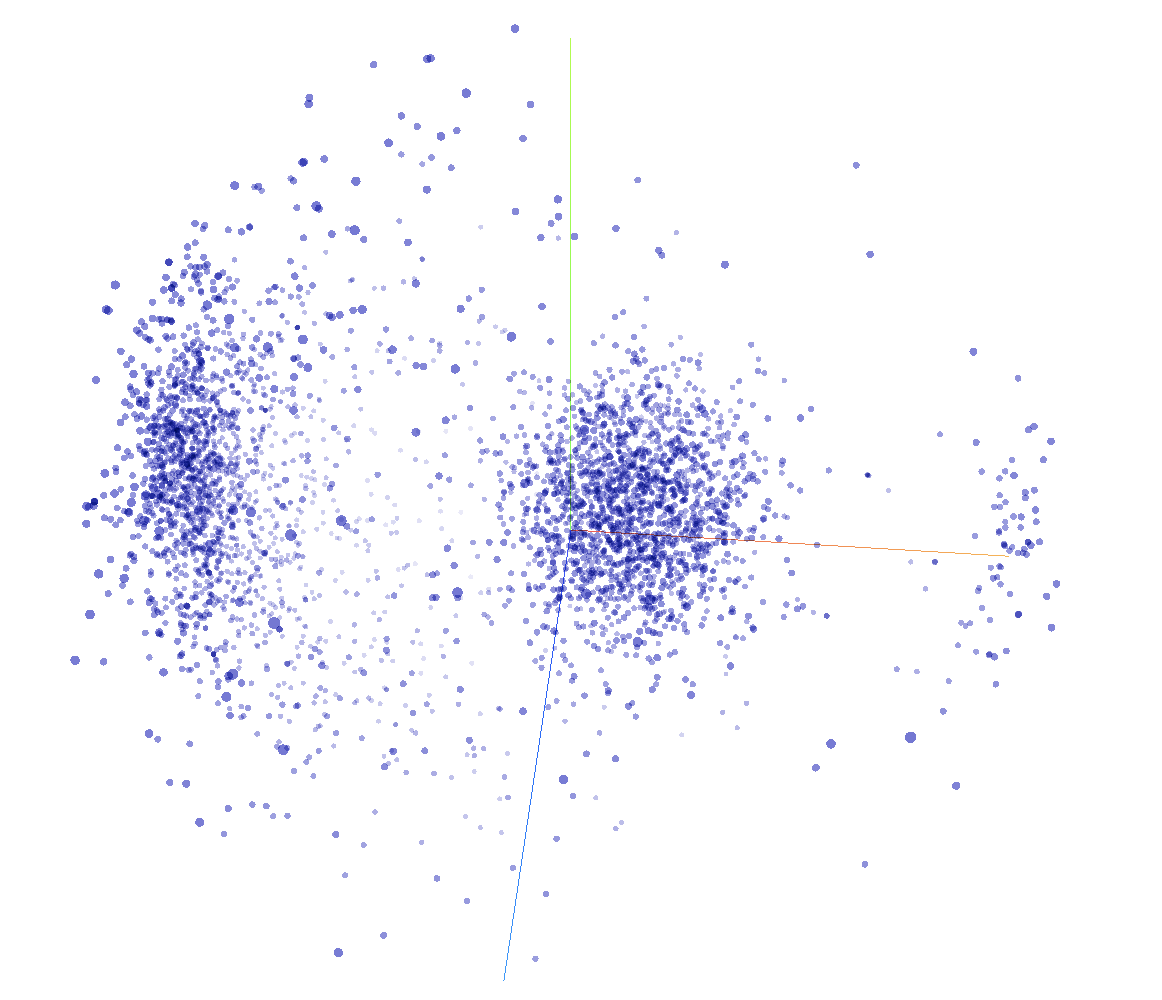} &
\hspace{-14pt}\includegraphics[width=0.3\linewidth]{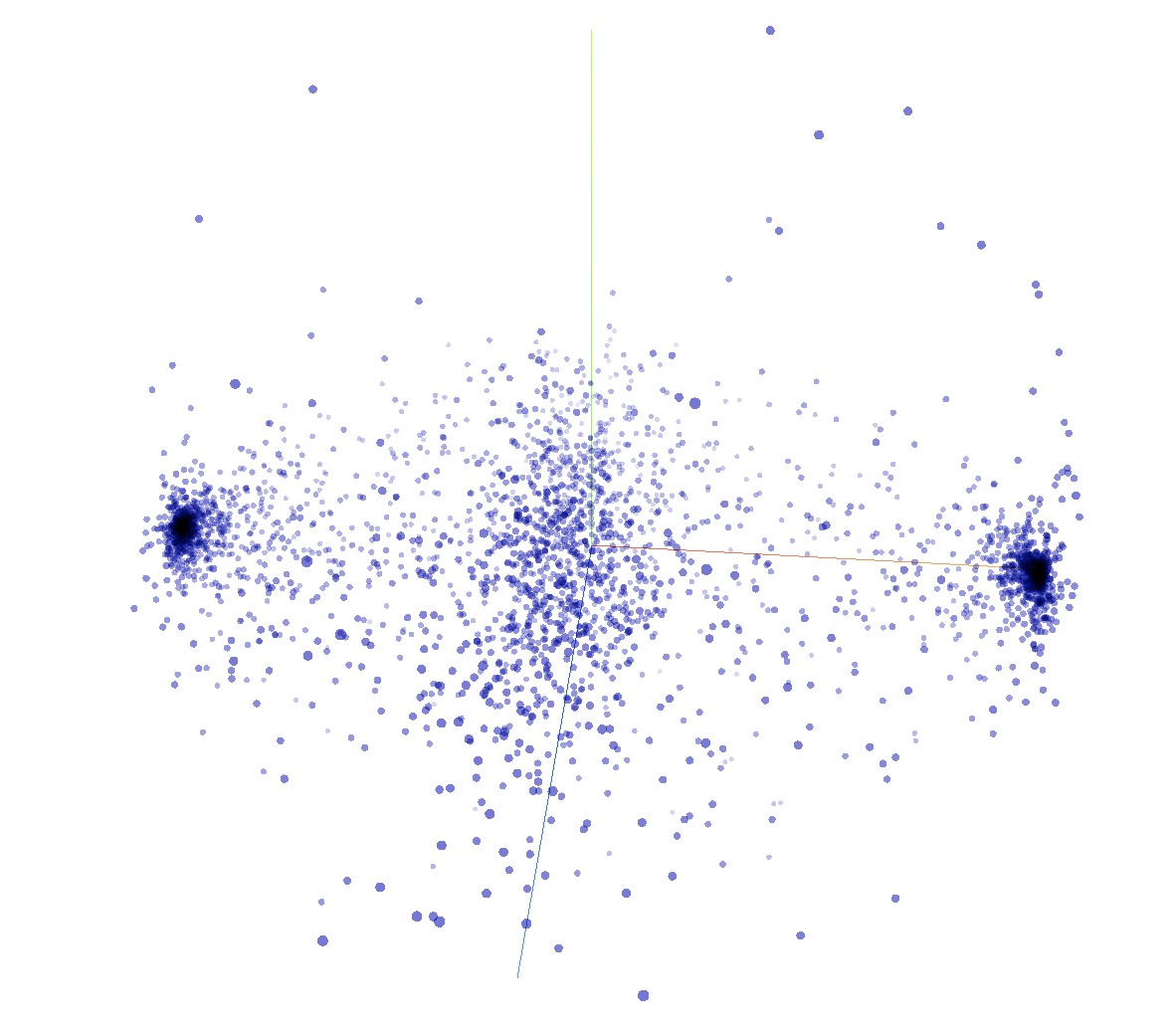} &
\hspace{-14pt}\includegraphics[width=0.3\linewidth]{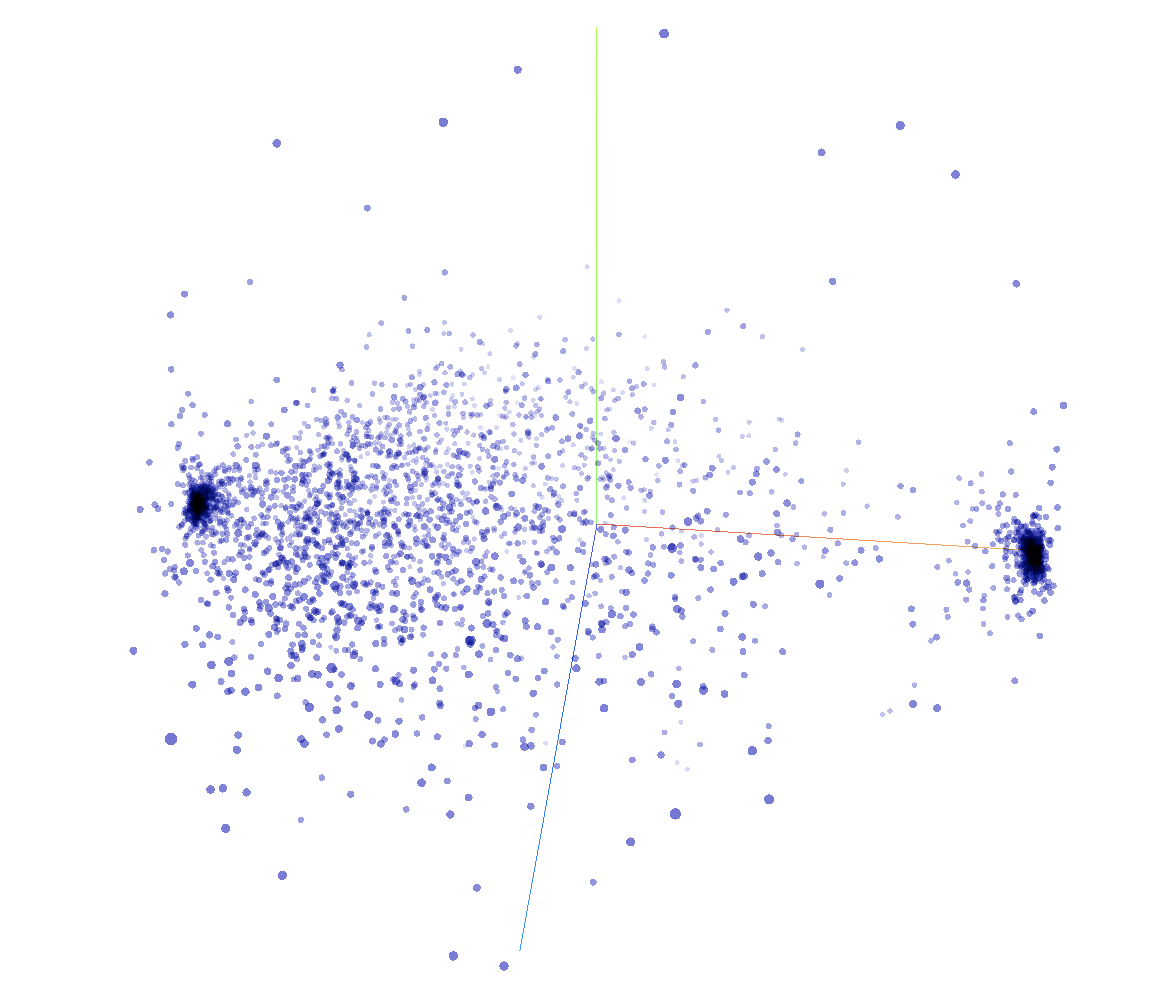}
\end{tabular}
\vspace{-10pt}\caption{Projecting 50-dimensional embeddings obtained by training a simple neural network without SSE (Left), and with SSE-Graph (Center) , SSE-SE (Right) into 3D space using PCA.}
\label{fig:pca}
\vspace{-10pt}\end{figure*}

\subsection{Simplified SSE with Complete Graph: SSE-SE}
One clear limitation of applying the SSE-Graph is that not every dataset comes with good-quality knowledge graphs on embeddings.
For those cases, we could assume there is a complete graph over all embeddings so there is a small transition probability between every pair of different embeddings:
%Moreover, storing a large knowledge graph and sampling replacement candidates can be a bottleneck in terms of both memory and speed in many cases. 
%Sometimes, even worse, we do not have a good knowledge on what the graphs should be for some embeddings. For example, we do not have a good understanding on what knowledge graphs for embeddings of words should be like. Due to all these mentioned limitations, 
%we come up with an alternative algorithm to replace SSE-Graph in such scenarios. We call it SSE-SE, an abbreviation for {\it Stochastic Shared Embeddings - Simple and Easy}. The core idea behind SSE-SE assumes there always exists a complete graph over all embeddings on top of SSE-Graph. SSE-SE is a special case of SSE-Graph with $\rho=1$ in Equation~\ref{eq:transition1}. In this case, Equation~\ref{eq:transition2} remains intact, but we can rewrite Equation~\ref{eq:transition1} as:
\begin{equation}\label{eq:sse-se}
    p(j, k|\Phi) = \frac{p_0}{N - 1}, \quad \forall 1 \leq k \neq j \leq N,
\end{equation} where $N$ is the size of the embedding table.
The SGD procedure in Algorithm~\ref{alg:see-graph} can still be applied and we call this algorithm SSE-SE (Stochastic Shared Embeddings - Simple and Easy).
It is worth noting that SSE-Graph and SSE-SE are applied to embeddings associated with not only input $x_i$ but also those with output $y_i$. Unless there are considerably many more embeddings than data points and model is significantly overfitting, normally $p_0 = 0.01$ gives reasonably good results.
%A very special case is that if the embeddings associated with output $y_i$ are one-hot encodings and cross entropy loss is used (or any loss linear in terms of $y_i$), 

Interestingly, we found that the SSE-SE framework is related to several techniques used in practice. 
For example, BERT pre-training unintentionally applied a method similar to SSE-SE to input $x_i$ by replacing the masked word with a random word. 
This would implicitly introduce an SSE layer for input $x_i$ in Figure~\ref{fig:train_test}, because now embeddings associated with input $x_i$ be stochastically mapped according to \eqref{eq:sse-se}. 
The main difference between this and SSE-SE is that it merely augments the input once, while SSE introduces randomization at every iteration, and we can also accommodate label embeddings.
%, and it is only for embeddings associated with input $x_i$ but not for output $y_i$, so it suffers from an avoidable performance degradation. 
In experimental Section~\ref{sec: bert}, we will show that SSE-SE would improve original BERT pre-training procedure as well as fine-tuning procedure. 

\subsection{Theoretical Guarantees}
% MAKE NOTATION CONSISTENT WITH MAIN BODY.

We explain why SSE can reduce the variance of estimators and thus leads to better generalization performance. 
For simplicity, we consider the SSE-graph objective \eqref{eq:single_obj} where there is no transition associated with the label $y_i$, and only the embeddings associated with the input $x_i$ undergo a transition.
When this is the case, we can think of the loss as a function of the $x_i$ embedding and the label, $\ell(\bo E[\bo j^i], y_i; \Theta)$.
We take this approach because it is more straightforward to compare our resulting theory to existing excess risk bounds.

%Thus, we may write $p(i,j) := p(i,j|\Theta)$ more succinctly, and let $i,j \in \{1,\ldots,n\}$ (in the case of SSE-graph $n$ is the number of nodes). 
%Let $n$ be the size of the 
The SSE objective in the case of only input transitions can be written as,
\begin{equation}
    S_n(\Theta) = \sum_i \sum_{\bo k} p(\bo j^i, \bo k) \cdot \ell(\bo E[\bo k],y_i|\Theta), 
\end{equation}
and there may be some constraint on $\Theta$.  
Let $\hat \Theta$ denote the minimizer of $S_n$ subject to this constraint. 
We will show in the subsequent theory that minimizing $S_n$ will get us close to a minimizer of $S(\Theta) = \mathbb E S_n(\Theta)$, and that under some conditions this will get us close to the Bayes risk.
We will use the standard definitions of empirical and true risk, $R_n(\Theta) = \sum_{i} \ell(x_i,y_i|\Theta)$ and $R(\Theta) = \mathbb E R_n(\Theta)$.

Our results depend on the following decomposition of the risk.  By optimality of $\hat \Theta$, 
\begin{equation}
    R(\hat{\Theta}) = S_n(\hat{\Theta}) + [R(\hat{\Theta}) - S(\hat{\Theta})] + [S(\hat{\Theta}) - S_n(\hat{\Theta})] \leq S_n(\Theta^\ast) + B(\hat{\Theta}) + \mathcal{E}(\hat{\Theta})
\end{equation}
where $B(\Theta) = | R(\Theta) - S(\Theta) |$,  and $E(\Theta) = | S(\Theta) - S_n(\Theta) |$. 
We can think of $B(\Theta)$ as representing the bias due to SSE, and $E(\Theta)$ as an SSE form of excess risk.
Then by another application of similar bounds,
\begin{equation}
R(\hat{\Theta}) \le R(\Theta^\ast) + B(\hat \Theta) + B(\Theta^\ast) + E(\hat \Theta) + E(\Theta^\ast).
\end{equation}
The high level idea behind the following results is that when the SSE protocol reflects the underlying distribution of the data, then the bias term $B(\Theta)$ is small, and if the SSE transitions are well mixing then the SSE excess risk $E(\Theta)$ will be of smaller order than the standard Rademacher complexity.
This will result in a small excess risk.

%     &= R(\Theta^\ast) + [S_n(\Theta^\ast) - S(\Theta^\ast)] + [S(\Theta^\ast) - R(\Theta^\ast)] + B(\hat{\Theta}) + \mathcal{E}(\hat{\Theta}) \\
%     &= R(\Theta^\ast) + B(\Theta^\ast) + B(\hat{\Theta}) + \mathcal{E}(\Theta^\ast) + \mathcal{E}(\hat{\Theta}) \\
%     &\leq R(\Theta^\ast) + 2 \sup_{\Theta}  B(\Theta) + 2 \sup_{\Theta} \mathcal{E}(\Theta)
% \end{align}, where

\begin{theorem}
\label{thm:rademacher}
Consider SSE-graph with only input transitions.
Let $L(\bo E[\bo j^i]) = \mathbb E_{Y | X = x^i} \ell(\bo E[\bo j^i],Y| \Theta)$ be the expected loss conditional on input $x^i$ and $e(\bo E[\bo j^i],y | \Theta) = \ell(\bo E[\bo j^i],y| \Theta) - L(\bo E[\bo j^i]| \Theta)$ be the residual loss.
Define the conditional and residual SSE empirical Rademacher complexities to be 
   \begin{align}
\label{eq:rade_1}
&\rho_{L,n} = \mathbb E_{\sigma} \sup_{\Theta} \left| \sum_{i} \sigma_{i} \sum_{\bo k} p(\bo j^i,\bo k) \cdot L(\bo E[\bo k] | \Theta) \right|, \\
\label{eq:rade_2}
&\rho_{e,n} = \mathbb E_{\sigma} \sup_{\Theta} \left| \sum_{i} \sigma_{i} \sum_{\bo k} p(\bo j^i,\bo k) \cdot e(\bo E[\bo k], y_i; \Theta) \right|,
% \\
% \label{eq:rade_3}
% &\rho_{2,n} = \mathbb E_{\sigma, \sigma'} \sup_\Theta \left| \sum_i \sum_{\bo k} (\sigma_i + \sigma_i') p(\bo j^i,\bo k) \ell(\bo E[\bo k],y_i ; \Theta)\right|,
\end{align}
respectively where $\sigma$ is a Rademacher $\pm 1$ random vectors in $\mathbb R^n$.
Then we can decompose the SSE empirical risk into 
\begin{equation}
\label{eq:excess_risk_bd}
\mathbb E \sup_{\Theta} |S_n(\Theta) - S(\Theta)| \le 2 \mathbb E [\rho_{L,n} + \rho_{e,n}].
\end{equation}
% Furthermore, suppose that $p$ is doubly stochastic (i.e.~$\sum_i p(i,j) = \sum_j p(i,j) = 1$), as in SSE-SE, then
% \begin{equation}
% \label{eq:rademacher_bd}
% \mathbb E[\rho_{2,n}] \le 2 \mathbb E \sup_{\Theta} \left| \sum_i \sigma_i l(x_i, y_i; \Theta) \right|,
% \end{equation}
% which is (twice) the standard empirical Rademacher complexity.
\end{theorem}

\begin{remark}
The transition probabilities in \eqref{eq:rade_1}, \eqref{eq:rade_2} act to smooth the empirical Rademacher complexity.  To see this, notice that we can write the inner term of \eqref{eq:rade_1} as $(P \sigma)^\top L$, where we have vectorized $\sigma_i, L(x_i;\Theta)$ and formed the transition matrix $P$.  Transition matrices are contractive and will induce dependencies between the Rademacher random variables, thereby stochastically reducing the supremum.  
In the case of no label noise, namely that $Y | X$ is a point mass, $e(x,y;\Theta) = 0$, and $\rho_{e,n} = 0$.  
The use of $L$ as opposed to the losses, $\ell$, will also make $\rho_{L,n}$ of smaller order than the standard empirical Rademacher complexity.
We demonstrate this with a partial simulation of $\rho_{L,n}$ on the Movielens1m dataset in Figure~\ref{fig:sim} of the Appendix.
%In this case, the factor of $2$ in the \eqref{eq:rademacher_bd} can be dropped, and the notion of complexity is a strict improvement.
\end{remark}

\begin{theorem}
\label{thm:main}
Let the SSE-bias be defined as
\[
\mathcal B = \sup_{\Theta} \left| \mathbb E \left[ \sum_i \sum_{\bo k} p(\bo j^i, \bo k) \cdot \left( \ell(\bo E[\bo k],y_i|\Theta) - \ell(\bo E[\bo j^i],y_i|\Theta) \right) \right] \right|.
\]
Suppose that $0 \le \ell(.,.;\Theta) \le b$ for some $b > 0$,
then 
\[
\mathbb P \left\{ R(\hat \Theta) > R(\Theta^*) + 2 \mathcal B + 4 \mathbb E [\rho_{L,n} + \rho_{e,n}] + \sqrt n u \right\} \le e^{-\frac{u^2}{2 b^2}}.
\]
\end{theorem}

% \begin{figure*}
% \begin{tabular}{ccc}
% \hspace{-8pt}\includegraphics[width=0.34\linewidth]{nips/mf_pca.png} &
% \hspace{-14pt}\includegraphics[width=0.34\linewidth]{nips/sse_graph_pca.png} &
% \hspace{-14pt}\includegraphics[width=0.34\linewidth]{nips/sse_pca.png}
% \end{tabular}
% \vspace{-10pt}\caption{.TODO.}
% \label{fig:sse-speed}
% \vspace{-10pt}\end{figure*}

% \begin{table}
%   \caption{Compare SSE-Graph and SSE-SE (users with at least 1 + 10 ratings)}
%   \label{tb:sse-se_sse-graph1}
%   \centering
%  %\resizebox{\columnwidth}{!}{
%   \begin{tabular}{cccccccccc}
%     \toprule
%     &\multicolumn{4}{c}{Movielens1m} & \multicolumn{4}{c}{Movielens10m}  \\
%     \cmidrule(r){2-5}   \cmidrule(r){6-9}
%     Model    & RMSE & $\rho$ & $p_{u}$ & $p_{i}$ & RMSE &  $\rho$ & $p_{u}$ & $p_{i}$\\
%     \midrule
%     ALS-MF      & 0.9319 &  -    & - &  - &1.8315 &  - & -   &  -\\

% Graph Laplacian + ALS-MF   & 0.9023  &  - &  -    & -  & 1.8251 &  - & -   &  -\\
%     \midrule
%     SGD-MF            & 0.8979 &     - & - & - & 1.7542 & -   & -   &  -\\
%     SSE-Graph + SGD-MF & \textbf{0.8957}   &  1000  & 0 & 0.005 & \textbf{1.7494}&  100  & 0   &  0.01 \\
%     SSE-SE + SGD-MF    & 0.8960& - & 0 & 0.005 & 1.7499& - & 0   &  0.01\\
%     \bottomrule
%   \end{tabular}
% %}
% \end{table}

\begin{table}
 \caption{Compare SSE-Graph and SSE-SE against ALS-MF with Graph Laplacian Regularization. The $p_u$ and $p_i$ are the SSE probabilities for user and item embedding tables respectively, as in \eqref{eq:sse-se}. Definitions of $\rho_u$ and $\rho_i$ can be found in \eqref{eq:transition1}. Movielens10m does not have user graphs. %{\color{red}(Cho: how about organizing this table into two parts 1) Algorithms using graph (Grahp laplacian, SSE-Graph) 2) Algorithms without graph (MF, MF+dropout(do we want to comapre with dropout? You are comparing it in BPR so maybe we should add it to all the comparisons? ), MF+SSE}}
% {\color{red}(Cho: need to say what is $\rho_u, \rho_i$. )
}
  \label{tb:sse-se_sse-graph2}
  \centering
 \resizebox{0.8\columnwidth}{!}{
  \begin{tabular}{cccccccccccc}
    \toprule
    &\multicolumn{5}{c}{Movielens1m} & \multicolumn{5}{c}{Movielens10m}  \\
    \cmidrule(r){2-6}   \cmidrule(r){7-11}
    Model    & RMSE & $ \rho_{u}$& $\rho_{i}$ & $p_{u}$ & $p_{i}$ & RMSE & $ \rho_{u}$& $\rho_{i}$ & $p_{u}$ & $p_{i}$\\
    \midrule
    % ALS-MF      & 1.1565& - &  -    & - &  - & 2.1185&- &  - & -   &  -\\
    SGD-MF            & 1.0984 & - &     - & - & - & 1.9490 & - & -   & -   &  -\\
  
        % \midrule
%GCN (2 hidden layers)& 1.0346 & - &  - &  -    & -  &  1.8632&- &  - & -   &  -\\
    Graph Laplacian + ALS-MF   & 1.0464 & - &  - &  -    & -  & 1.9755&- &  - & -   &  -\\
    
%  SSE-Graph + SGD-MF & \textbf{1.0172}   &  1000  & 0 & 0.005 & \textbf{1.9019}& 500  & 0.01   &  0.01 \\
SSE-Graph + SGD-MF & \textbf{1.0145}& 500   &  200  & 0.005 & 0.005 & \textbf{1.9019}& 1 & 500  & 0.01   &  0.01 \\
  SSE-SE + SGD-MF    & 1.0150 & 1 & 1 & 0.005 & 0.005 & 1.9085& 1 & 1& 0.01   &  0.01\\
    \bottomrule
  \end{tabular}
}
 \vspace{-7pt}
\end{table}

%\subsubsection{Evaluate SSE-SE Performances}

\begin{table}
  \caption{SSE-SE outperforms Dropout for Neural Networks with One Hidden Layer such as Matrix Factorization Algorithm regardless of dimensionality we use. $p_s$ is the SSE probability for both user and item embedding tables and $p_d$ is the dropout probability.}
  \label{tb:mf}
  \centering
 \resizebox{0.7\columnwidth}{!}{
  \begin{tabular}{cccccccccc}
    \toprule
    & \multicolumn{3}{c}{Douban} & \multicolumn{3}{c}{Movielens10m}  & \multicolumn{3}{c}{Netflix}          \\
    \cmidrule(r){2-4} \cmidrule(r){5-7}  \cmidrule(r){8-10} 
    Model     & RMSE & $p_{d}$ & $p_{s}$ & RMSE & $p_{d}$ & $p_{s}$ & RMSE & $p_{d}$ & $p_{s}$\\
    \midrule
    MF   & 0.7339 & - & - & 0.8851  & -  & - & 0.8941 & -  & - \\
    Dropout + MF   & 0.7296 & 0.1 & - & 0.8813  & 0.1  & - &  0.8897 & 0.1  & -   \\
    SSE-SE + MF   & 0.7201 & - & 0.008  & 0.8715 & -   & 0.008  &  0.8842 &  -   & 0.008 \\
    SSE-SE + Dropout + MF   & \textbf{0.7185} & 0.1 & 0.005  & \textbf{0.8678} & 0.1  & 0.005   & \textbf{0.8790} & 0.1  & 0.005 \\
    \bottomrule
  \end{tabular}
}
 \vspace{-4pt}
\end{table}

\begin{table}
  \caption{SSE-SE outperforms dropout for Neural Networks with One Hidden Layer such as Bayesian Personalized Ranking Algorithm regardless of dimensionality we use. We report the metric precision for top $k$ recommendations as $P@k$.}
  \label{tb:bpr}
  \centering
 \resizebox{0.75\columnwidth}{!}{
  \begin{tabular}{cccccccccc}
    \toprule
    & \multicolumn{3}{c}{Movielens1m}   & \multicolumn{3}{c}{Yahoo Music}   & \multicolumn{3}{c}{Foursquare}        \\
    \cmidrule(r){2-4} \cmidrule(r){5-7}  \cmidrule(r){8-10}  
    Model     & $P@1$ & $P@5$ & $P@10$ & $P@1$ & $P@5$ & $P@10$ & $P@1$ & $P@5$ & $P@10$ \\
    \midrule
    % Weighted-MF & 0.5469 & 0.4942 &  0.4612 & 0.3908 & 0.3102 & 0.2701 & 0.0218 & 0.0155 & 0.0141  \\
    SQL-Rank (2018) & \textbf{0.7369} & 0.6717 & 0.6183  &  \textbf{0.4551}   &  \textbf{0.3614}   & \textbf{0.3069}  & 0.0583  & 0.0194   & \textbf{0.0170}  \\
    \midrule
    BPR   & 0.6977 & 0.6568 &  0.6257   & 0.3971 & 0.3295 &   0.2806 & 0.0437 & 0.0189 &   0.0143  \\
    Dropout + BPR   & 0.7031 & 0.6548 &  0.6273  & 0.4080 & 0.3315 &  0.2847 & 0.0437 & 0.0184 &  0.0146  \\
    SSE-SE + BPR   & 0.7254 & \textbf{0.6813} &  \textbf{0.6469}   & 0.4297 & 0.3498 & 0.3005 & \textbf{0.0609} & \textbf{0.0262} & 0.0155  \\
    \bottomrule
  \end{tabular}
}
\vspace{-10pt}
\end{table}

\begin{remark}
The price for `smoothing' the Rademacher complexity in Theorem \ref{thm:rademacher} is that SSE may introduce a bias.  This will be particularly prominent when the SSE transitions have little to do with the underlying distribution of $Y,X$.  On the other extreme, suppose that $p(\bo j, \bo k)$ is non-zero over a neighborhood $\mathcal N_{\bo j}$ of $\bo j$, and that for data $x',y'$ with encoding $\bo k \in \mathcal N_{\bo j}$, $x', y'$ is identically distributed with $x_i,y_i$, then $\mathcal B = 0$.  In all likelihood, the SSE transition probabilities will not be supported over neighborhoods of iid random pairs, but with a well chosen SSE protocol the neighborhoods contain approximately iid pairs and $\mathcal B$ is small.
\end{remark}

\section{Experiments}\label{sec:experiments}
% A central question here is what methods we want to compare against. Listing a bunch of methods does not help support our arguments. 
% One possibility is that we can limit our comparisons to parameter sharing and state we can incorporate weight decay and dropout better than other embedding regularization methods. 
We have conducted extensive experiments on 6 tasks, including 3 recommendation tasks (explicit feedback, implicit feedback and sequential recommendation) and 3 NLP tasks (neural machine translation, BERT  pre-training, and BERT fine-tuning for sentiment classification) and found that our proposed SSE can effectively improve generalization performances on a wide variety of tasks. 
Note that the details about datasets and parameter settings can be found in the appendix.  %detailed experimental settings can be found in the 

% \vspace{-.3cm}
\subsection{Neural Networks with One Hidden Layer (Matrix Factorization and BPR)}\label{sec:simple_nn}
% Results on MF and BPR. First, we show the results of SSE with graphs before showing the results of simplified SSE algorithm. For the graph comparisons, we may want to use Graph regularized alternating least squares and graph convolutional matrix completion as baselines. We analyze what impact SSE has on embedding distributions. Use tool \url{https://projector.tensorflow.org/} to visualize embeddings via PCA or t-SNE. Add plots here or in intro.

% Omit BPR if running out of space. Just state that the loss we used here does not matter. Even with a pairwise ranking loss, we observe similar results. Also, it does not matter how many different embeddings we draw from each embedding look-up table for each SGD update. 

%We start by examining some simple neural networks with one hidden layer only. 
Matrix Factorization Algorithm (MF) \cite{mnih2008probabilistic} and Bayesian Personalized Ranking Algorithm (BPR) \cite{rendle2009bpr} can be viewed as neural networks with one hidden layer (latent features) and are quite popular in recommendation tasks.
%, motivating many deeper recommendation neural networks thereafter. 
MF uses the squared loss designed for explicit feedback data while BPR uses the pairwise ranking loss designed for implicit feedback data. 
%The main difference between the two is the loss function they utilize: MF uses the squared loss while BPR uses the pairwise ranking loss. One is suited for explicit feedback data such as user ratings and the other is suited for implicit feedback data such as user clicks.

First, we conduct experiments on two explicit feedback datasets: Movielens1m and Movielens10m.
For these datasets, we can construct graphs based on actors/actresses starring the movies. % (see more details in Appendix \ref{app:mf}). 
We compare SSE-graph and the popular Graph Laplacian Regularization (GLR) method~\cite{rao2015collaborative} in Table~\ref{tb:sse-se_sse-graph2}. The results  show that SSE-graph consistently outperforms GLR. 
This indicates that our SSE-Graph has greater potentials over graph Laplacian regularization as we do not explicitly penalize the distances across embeddings, but rather we implicitly penalize the effects of similar embeddings on the loss.
Furthermore, we show that even without existing knowledge graphs of embeddings, our SSE-SE  performs only slightly worse than SSE-Graph but still much better than GLR and MF. 

In general, SSE-SE is a good alternative when graph information is not available.
We then show that our proposed SSE-SE can be used together with standard regularization techniques such as dropout and weight decay to improve recommendation results regardless of the loss functions and dimensionality of embeddings.  This is evident in Table~\ref{tb:mf} and Table~\ref{tb:bpr}. With the help of SSE-SE, BPR can perform better than the state-of-art listwise approach SQL-Rank \cite{wu2018sql} in most cases. We include the optimal SSE parameters in the table for references and leave out other experiment details to the appendix. In the rest of the paper, we would mostly focus on SSE-SE as we do not have high-quality graphs of embeddings on most datasets.

\begin{table}
  \caption{SSE-SE has two tuning parameters: probability $p_x$ to replace embeddings associated with input $x_i$ and probability $p_y$ to replace embeddings associated with output $y_i$. We use the dropout probability of $0.1$, weight decay of $1e^{-5}$, and learning rate of $1e^{-3}$ for all experiments.}
  \label{tb:sasrec}
  \centering
 \resizebox{0.75\columnwidth}{!}{
  \begin{tabular}{ccccccc}
    \toprule
    & \multicolumn{2}{c}{Movielens1m } & Dimension& $\#$ of Blocks  & \multicolumn{2}{c}{SSE-SE Parameters}             \\
    \cmidrule(r){2-3} \cmidrule(r){4-5}  \cmidrule(r){6-7}
    Model     & NDCG$@10$ & Hit Ratio$@10$ & $d$ & $b$ & $p_x$  & $p_y$ \\
    \midrule
   % SASRec            & 0.5910 & 0.8209 & 50 & 2  & -  & - \\
    SASRec   & 0.5941 & 0.8182 & 100 & 2  & -  & -\\
    % SASRec   & 0.5857 & 0.8088 & 200 & 2  & -  & -\\
    SASRec   & 0.5996 & 0.8272 & 100 & 6  & -  & -\\
    \midrule
   % SSE-SE + SASRec   & 0.6058 & 0.8305 & 50  & 2 & 0  & 0.1  \\
    SSE-SE + SASRec   & 0.6092 & 0.8250 & 100  & 2 & 0.1  & 0  \\
    SSE-SE + SASRec   & 0.6085 & 0.8293 & 100  & 2 & 0  & 0.1  \\
    SSE-SE + SASRec   & 0.6200 & 0.8315 & 100  & 2 & 0.1  & 0.1  \\
    \midrule
    SSE-SE + SASRec   & \bfseries{0.6265} & \bfseries{0.8364} & 100  & 6 & 0.1  & 0.1  \\
    %  \midrule
    % SSE-SE + SASRec   & 0.6367 & 0.8474 & 100  & 6 & 0 (item)  & 0.1  \\
    % SSE-SE + PT       & 0.6392 & 0.8535 & 50 + 50  & 6 & 0.92 (user) + 0 (item)  & 0.1 \\
    % SSE-SE + NLP-SASRec   & 0.6383 & 0.8490 & 100  & 6 & 0 (item) & 0.1  \\
    % SSE-SE + NLP-PT          & 0.6412 & 0.8538 & 50 + 25 + 25  & 6 & 0.92 (user) + 0 (item) & 0.1 \\
    \bottomrule
  \end{tabular}
}
\vspace{-4pt}
\end{table}

\subsection{Transformer Encoder Model for Sequential Recommendation}\label{sec:sasrec}
% One thing we definitely should add is that SSE makes the model's generalization power not sensitive to the number of hidden units used for embeddings. SSE enables us to use a larger number of hidden units for embedding representations so that we will benefit more from using a deeper neural networks. Show the results when we extend from 2 block, 7 layers neural nets to 6 blocks, 19 layers neural nets. Remember to state that we use layer normalization instead of batch normalization as the former gives favorable results. 

% Cross Entropy loss is used here and again the loss does not change anything. We also use a different optimizer Adam and draw many more embeddings from the items' embedding look-up table. None of the differences matter and our algorithm is the only one working well when we try to increase the number of hidden units for embeddings. All the existing regularization techniques failed. Add a line plot here.

SASRec \cite{kang2018self} is the state-of-the-arts algorithm for sequential recommendation task. It applies the transformer model \cite{vaswani2017attention}, where a sequence of items purchased by a user can be viewed as a sentence in transformer, and next item prediction is equivalent to next word prediction in the language model.
%We use dropout probability of 0.1, weight decay of $1e^{-5}$ and learning rate of $1e^{-3}$ for all experiments. The embeddings of items are associated with both input $x_i$ and output $y_i$ here as the inner product of item embeddings are calculated. We use SSE probability of 0.1 for SSE-SE. 
In Table~\ref{tb:sasrec}, we perform SSE-SE on input embeddings ($p_x=0.1$, $p_y=0$), output embeddings ($p_x=0.1$, $p_y=0$) and both embeddings ($p_x=p_y=0.1$), and observe that all of them significantly improve over state-of-the-art SASRec ($p_x=p_y=0$). 
%We find that in Table~\ref{tb:sasrec} that using SSE-SE for embeddings associated with input and/or output only will improve ranking results: NDCG@10 from 0.5941 to 0.6092 and 0.6085 respectively and Hit Ratio@10 from 0.8182 to 0.8250 and 0.8293 respectively. But when we use SSE-SE for embeddings associated with both input and output, we will see the improvement is larger: 0.6212 for NDCG@10 and 0.8324 for Hit Ratio@10. 
The regularization effects of SSE-SE is even more obvious when we increase the number of self-attention blocks from 2 to 6, as this will lead to a more sophisticated model with many more parameters. This leads to the model overfitting terribly even with dropout and weight decay. We can see in Table~\ref{tb:sasrec} that when both methods use dropout and weight decay, SSE-SE + SASRec is doing much better than SASRec without SSE-SE.
%: NDCG@10 is 0.6265 versus 0.5996 and Hit Ratio@10 is 0.8364 versus 0.8272.  

\begin{table}
  \caption{Our proposed SSE-SE helps the Transformer achieve better BLEU scores on English-to-German in 10 out of 11 newstest data between 2008 and 2018.}
  \label{tb:nmt}
  \centering
  \resizebox{0.8\columnwidth}{!}{
  \begin{tabular}{cccccccccccc}
    \toprule
    & \multicolumn{11}{c}{Test BLEU}                   \\
    \cmidrule(r){2-12} 
    Model     & 2008 & 2009 & 2010 & 2011 & 2012   & 2013  & 2014 & 2015 & 2016 & 2017 & 2018 \\
    \midrule
    Transformer            & 21.0 &      20.7  & 22.7 
                        & 20.6 & 20.6 & \bfseries{25.3}   & 26.2 & 28.4 & 32.1 & 27.2 &  38.8  \\
    SSE-SE + Transformer    & \bfseries{21.4} & \bfseries{21.1} & \bfseries{23.0} 
                        & \bfseries{21.0} & \bfseries{20.8}  & 25.2  & \bfseries{27.2}  & \bfseries{29.2} & \bfseries{33.1} & \bfseries{27.9} &  \bfseries{39.9} \\
    \bottomrule
  \end{tabular}
  }
  \vspace{-12pt}
\end{table}

\vspace{-.3cm}
\subsection{Neural Machine Translation}\label{sec:nmt}
% {\color{red}(Cho: maybe switch this with BERT result, since there are more datasets in here.)}
We use the transformer model \cite{vaswani2017attention} as the backbone for our experiments. 
The baseline model is the standard 6-layer transformer architecture and we apply SSE-SE to both encoder, and decoder by replacing corresponding vocabularies' embeddings in the source and target sentences. We trained on the standard WMT 2014 English to German dataset which consists of roughly 4.5 million parallel sentence pairs and tested on WMT 2008 to 2018 news-test sets. 
We use the OpenNMT implementation in our experiments. We use the same dropout rate of 0.1 and label smoothing value of 0.1 for the baseline model and our SSE-enhanced model. 
% Both models have dimensionality of embeddings as $d = 512$. When decoding, we use beam search with beam size of 4 and length penalty of 0.6 and replace unknown words using attention. For both models, we average last 5 checkpoints (we save checkpoints every 10,000 step) and evaluate the model's performances on the test datasets using BLEU scores \cite{post-2018-call}. 
The only difference between the two models is whether or not we use our proposed SSE-SE with $p_0 = 0.01$ in \eqref{eq:sse-se} for both encoder and decoder embedding layers. We evaluate both models' performances on the test datasets using BLEU scores \cite{post-2018-call}.

%The control group is the standard transformer encoder-decoder architecture with self-attention. In experiment group, we apply SSE-SE towards both encoder and decoder by replacing corresponding vocabularies' embeddings in the source and target sentences. We trained on the standard WMT 2014 English to German dataset which consists of roughly 4.5 million parallel sentence pairs and tested on WMT 2008 to 2018 news-test sets. Sentences were encoded into 32,000 tokens using byte-pair encoding. We use the SentencePiece, OpenNMT and SacreBLEU implementations in our experiments. We trained the 6-layer transformer base model on a single machine with 4 NVIDIA V100 GPUs for 20,000 steps. We use the same dropout rate of 0.1 and label smoothing value of 0.1 for the baseline model and our SSE-enhanced model. Both models have dimensionality of embeddings as $d = 512$. When decoding, we use beam search with beam size of 4 and length penalty of 0.6 and replace unknown words using attention. For both models, we average last 5 checkpoints (we save checkpoints every 10,000 step) and evaluate the model's performances on the test datasets using BLEU scores. The only difference between the two models is whether or not we use our proposed SSE-SE with $p = 0.01$ in Equation~\ref{eq:sse-se} for both encoder and decoder embedding layers.
%{\color{red}(put some of these details in appendix .)}

We summarize our results in Table~\ref{tb:nmt} and find that SSE-SE helps improving accuracy and BLEU scores on both dev and test sets in 10 out of 11 years from 2008 to 2018. In particular, on the last 5 years' test sets from 2014 to 2018, the transformer model with SSE-SE improves BLEU scores by 0.92 on average when compared to the baseline model without SSE-SE.

% \vspace{-.3cm}
\subsection{BERT for Sentiment Classification}\label{sec: bert}
% Can we show SSE improves pre-training stage of Bert? We can try to apply SSE to pre-training on movie reviews data and see how this helps movie sentiment tasks. Use wordNet instead of uniformly random? 

% Pretraining and Fine tuning. SSE can be applied to both. 3 cases: baseline, SSE applied to pretraining only, fine-tuning only and both.

BERT's model architecture \cite{devlin2018bert} is a multi-layer bidirectional Transformer encoder based on the Transformer model in neural machine translation.
%The main architecture difference between BERT and the model in previous section are two fold: first, BERT is bidirectional so there are no masked attention matrices; second, the loss is quite different, in pre-training stage BERT utilizes the sum of two losses: one for the masked language models and the other for the next sentence prediction.
Despite SSE-SE can be used for both pre-training and fine-tuning stages of BERT, we want to mainly focus on pre-training as fine-tuning bears more similarity to the previous section. We use SSE probability of 0.015 for embeddings (one-hot encodings) associated with labels and SSE probability of 0.015 for embeddings (word-piece embeddings) associated with inputs. One thing worth noting is that even in the original BERT model's pre-training stage, SSE-SE is already implicitly used for token embeddings. In original BERT model, the authors masked 15\% of words for a maximum of 80 words in sequences of maximum length of 512 and 10\% of the time replaced the [mask] token with a random token. That is roughly equivalent to SSE probability of 0.015 for replacing input word-piece embeddings. 

% We do not think the BERT authors explained the benefits accurately here. The authors stated their motivation of doing so is to mitigate a mismatch between pre-training and fine-tuning since [mask] token is never seen in fine-tuning stage. But we think it is due to the fact that replacing [mask] token with a random token has the same effects of replacing the tokens' embeddings as SSE-SE does and this would allow us to mitigate over-parameterization in embedding layer and to learn better embedding representations without reducing the number of parameters of the model.

We continue to pre-train Google pre-trained BERT model on our crawled IMDB movie reviews with and without SSE-SE and compare downstream tasks performances. In Table~\ref{tb:bert-imdb}, we find that SSE-SE pre-trained BERT base model helps us achieve the state-of-the-art results for the IMDB sentiment classification task, which is better than the previous best in \cite{howard2018universal}. We report test set accuracy of 0.9542 after fine-tuning for one epoch only. For the similar SST-2 sentiment classification task in Table~\ref{tb:bert-sst2}, we also find that SSE-SE can improve BERT pre-trains better. Our SSE-SE pre-trained model achieves 94.3\% accuracy on SST-2 test set after 3 epochs of fine-tuning while the standard pre-trained BERT model only reports 93.8 after fine-tuning. Furthermore, we show that SSE-SE with SSE probability 0.01 can also improve dev and test accuracy in the fine-tuning stage. If we are using SSE-SE for both pre-training and fine-tuning stage of the BERT base model, we can achieve 94.5\% accuracy on the SST-2 test set, approaching the 94.9\% accuracy by the BERT large model. We are optimistic that our SSE-SE can be applied to BERT large model as well in the future. 

% We also tried to use SSE-SE in an ongoing Kaggle competition (Jigsaw Unintended Bias in Toxicity Classification Competition) and find that SSE-SE can improve the ensemble model's generalization results, helping us rank top 5 in more than 1700 teams. We reported some results in Table~\ref{tb:ensemble} in Appendix but omitted most details.  
% {\color{red}(tricky to say since ranking might change. Also improvement is not very significant so maybe we can remove. )}

% SST2: 128 max sequence length and 3 epochs fine-tuning \\
% IMDB: 512 max sequence length and 1 epochs fine-tuning (2 epochs for google pre-trained model)
% both use learning rate of $2e^{-5}$, dropout probability of $0.1$

\begin{table}
  \caption{Our proposed SSE-SE applied in the pre-training stage on our crawled IMDB data improves the generalization ability of pre-trained IMDB model and helps the BERT-Base model outperform current SOTA results on the IMDB Sentiment Task after fine-tuning.}
  \label{tb:bert-imdb}
  \centering
 \resizebox{0.65\columnwidth}{!}{
  \begin{tabular}{cccc}
    \toprule
    & \multicolumn{3}{c}{IMDB Test Set}                \\
    \cmidrule(r){2-4}  
    Model     & AUC & Accuracy & F1 Score \\
    \midrule
    ULMFiT \cite{howard2018universal} & - & 0.9540 & - \\
    \midrule
    Google Pre-trained Model + Fine-tuning            & 0.9415 &  0.9415    & 0.9419  \\
    Pre-training + Fine-tuning   & 0.9518 & 0.9518 &  0.9523 \\
    (SSE-SE + Pre-training) + Fine-tuning   & \bfseries{0.9542}  & \bfseries{0.9542}    & \bfseries{0.9545}    \\
    \bottomrule
  \end{tabular}
  \vspace{-10pt}
 }
\end{table}

\begin{table}
  \caption{SSE-SE pre-trained BERT-Base models on IMDB datasets turn out working better on the new unseen SST-2 Task as well. }
%   Note that the first two rows result are copied from the BERT paper and our reported results are based on the BERT-Base models.}
  \label{tb:bert-sst2}
  \centering
\resizebox{0.8\columnwidth}{!}{
  \begin{tabular}{ccccc}
    \toprule
    & \multicolumn{3}{c}{SST-2 Dev Set} & \multicolumn{1}{c}{SST-2 Test Set}                 \\
    \cmidrule(r){2-4}  \cmidrule(r){5-5} 
    Model     & AUC & Accuracy & F1 Score & Accuracy (\%) \\
    \midrule 
    % BERT-Base & - &  -  & -  & 93.5 \\
    % BERT-Large  & - &  -  & -  & 94.9  \\
    % \midrule
    Google Pre-trained + Fine-tuning             & 0.9230 &      0.9232  & 0.9253  &   93.6  \\
    Pre-training + Fine-tuning   & 0.9265 & 0.9266 & 0.9281 &   93.8 \\
    (SSE-SE + Pre-training) + Fine-tuning   & 0.9276 &  0.9278  & 0.9295  & 94.3  \\
    (SSE-SE + Pre-training) + (SSE-SE + Fine-tuning) & \bfseries{0.9323} &  \bfseries{0.9323}  & \bfseries{0.9336}  & \bfseries{94.5} \\

    %Reviews Pre-training +  Fine-tuning (SSE)   &  &      &   &  &        &  \\
    %Reviews Pre-training (SSE) + SSE Fine-tuning (SSE)   &  &      &   &  &        &  \\
    \bottomrule
  \end{tabular}
}
\end{table}

\begin{figure*}
\centering
\begin{tabular}{cc}
\hspace{-8pt}
\includegraphics[width=0.46\linewidth]{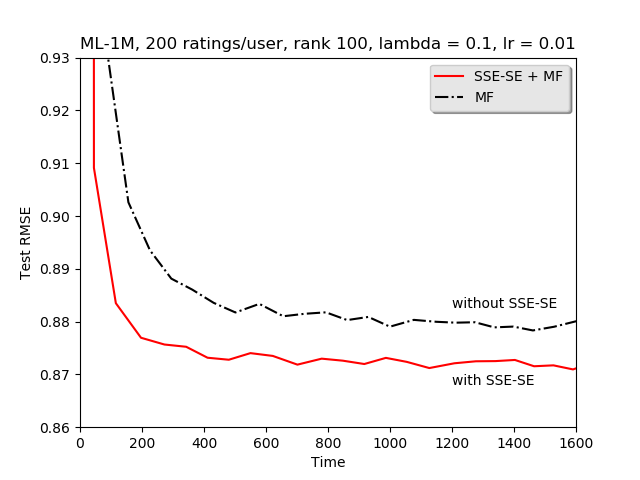} &
%\hspace{-8pt}
\includegraphics[width=0.46\linewidth]{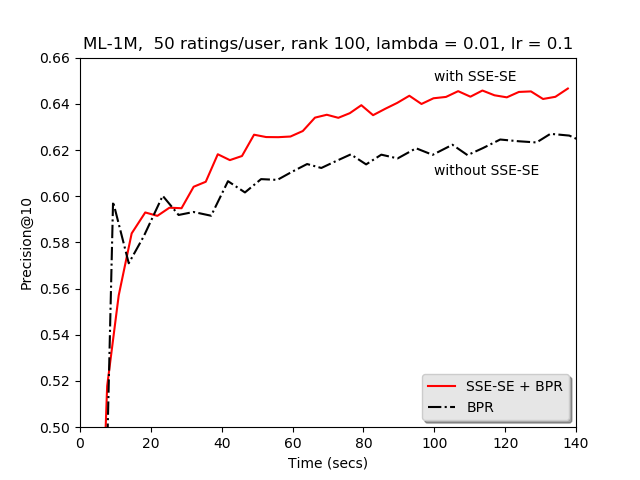} 
\end{tabular}
\vspace{-10pt}\caption{Compare Training Speed of Simple Neural Networks with One Hidden Layer, i.e. MF and BPR, with and without SSE-SE. }
\label{fig:sse-speed}
% \vspace{-10pt}
\end{figure*}

%\subsection{Speed and Convergence Comparisons}
% \vspace{-.3cm}
\subsection{Speed and Convergence Comparisons. }
% Put speed plots here. Two plots side by side.
In Figure~\ref{fig:sse-speed}, it is clear to see that our one-hidden-layer neural networks with SSE-SE are achieving much better generalization results than their respective standalone versions. One can also easily spot that SSE-version algorithms converge at much faster speeds with the same learning rate.

%  \vspace{-10pt}
\section{Conclusion} 
We have proposed Stochastic Shared Embeddings, which is a data-driven approach to regularization, that stands in contrast to brute force regularization such as Laplacian and ridge regularization.  Our theory is a first step towards explaining the regularization effect of SSE, particularly, by `smoothing' the Rademacher complexity.  The extensive experimentation demonstrates that SSE can be fruitfully integrated into existing deep learning applications.

{\bf Acknowledgement.}
Hsieh acknowledges the support of NSF IIS-1719097, Intel faculty award, Google Cloud and Nvidia.

\clearpage
\bibliography{neurips_2019}
\bibliographystyle{neurips_2019}

% [1] Alexander, J.A.\ \& Mozer, M.C.\ (1995) Template-based algorithms for
% connectionist rule extraction. In G.\ Tesauro, D.S.\ Touretzky and T.K.\ Leen
% (eds.), {\it Advances in Neural Information Processing Systems 7},
% pp.\ 609--616. Cambridge, MA: MIT Press.

% [2] Bower, J.M.\ \& Beeman, D.\ (1995) {\it The Book of GENESIS: Exploring
%   Realistic Neural Models with the GEneral NEural SImulation System.}  New York:
% TELOS/Springer--Verlag.

% [3] Hasselmo, M.E., Schnell, E.\ \& Barkai, E.\ (1995) Dynamics of learning and
% recall at excitatory recurrent synapses and cholinergic modulation in rat
% hippocampal region CA3. {\it Journal of Neuroscience} {\bf 15}(7):5249-5262.

\newpage
\section{Appendix}

For experiments in Section~\ref{sec:simple_nn}, we use Julia and C++ to implement SGD. For experiments in Section~\ref{sec:sasrec}, and Section~\ref{sec: bert}, we use Tensorflow and SGD/Adam Optimizer. For experiments in Section~\ref{sec:nmt}, we use Pytorch and Adam with noam decay scheme and warm-up. We find that none of these choices affect the strong empirical results supporting the effectiveness of our proposed methods, especially the SSE-SE. In any deep learning frameworks, we can introduce stochasticity to the original embedding look-up behaviors and easily implement SSE-Layer in Figure~\ref{fig:train_test} as a custom operator. 

\subsection{Neural Networks with One Hidden Layer}
\label{app:mf}
To run SSE-Graph, we need to construct good-quality knowledge graphs on embeddings. We managed to match movies in Movielens1m and Movielens10m datasets to IMDB websites, therefore we can extract plentiful information for each movie, such as the cast of the movies, user reviews and so on. For simplicity reason, we construct the knowledge graph on item-side embeddings using the cast of movies. Two items are connected by an edge when they share one or more actors/actresses. For user side, we do not have good quality graphs: we are only able to create a graph on users in Movielens1m dataset based on their age groups but we do not have any side information on users in Movielens10m dataset. When running experiments, we do a parameter sweep for weight decay parameter and then fix it before tuning the parameters for SSE-Graph and SSE-SE. We utilize different $\rho$ and $p$ for user and item embedding tables respectively. The optimal parameters are stated in Table~\ref{tb:sse-se_sse-graph2} and Table~\ref{tb:mf}. We use the learning rate of 0.01 in all SGD experiments.

In the first leg of experiments, we examine users with fewer than 60 ratings in Movielens1m and Movielens10m datasets. In this scenario, the graph should carry higher importance. One can easily see from Table~\ref{tb:sse-se_sse-graph2} that without using graph information, our proposed SSE-SE is the best performing matrix factorization algorithms among all methods, including popular ALS-MF and SGD-MF in terms of RMSE. With Graph information, our proposed SSE-Graph is performing significantly better than the Graph Laplacian Regularized Matrix Factorization method. This indicates that our SSE-Graph has great potentials over Graph Laplacian Regularization as we do not explicitly penalize the distances across embeddings but rather we implicitly penalize the effects of similar embeddings on the loss.

In the second leg of experiments, we remove the constraints on the maximum number of ratings per user. We want to show that SSE-SE can be a good alternative when graph information is not available. We follow the same procedures in \cite{wu2017large, wu2018sql}. In Table~\ref{tb:mf}, we can see that SSE-SE can be used with dropout to achieve the smallest RMSE across Douban, Movielens10m, and Netflix datasets. In Table~\ref{tb:bpr}, one can see that SSE-SE is more effective than dropout in this case and can perform better than STOA listwise approach SQL-Rank \cite{wu2018sql} on 2 datasets out of 3. 

In Table~\ref{tb:mf}, SSE-SE has two tuning parameters: probability $p_u$ to replace embeddings associated with user-side embeddings and probability $p_i$ to replace embeddings associated with item side embeddings because there are two embedding tables. But here for simplicity, we use one tuning parameter $p_{s} = p_u = p_i$. We use dropout probability of $p_{d}$, dimension of user/item embeddings $d$, weight decay of $\lambda$ and learning rate of $0.01$ for all experiments, with the exception that the learning rate is reduced to $0.005$ when both SSE-SE and Dropout are applied. For Douban dataset, we use $d = 10, \lambda = 0.08$. For Movielens10m and Netflix dataset, we use $d = 50, \lambda = 0.1$.

\subsection{Neural Machine Translation}
We use the transformer model \cite{vaswani2017attention} as the backbone for our experiments. The control group is the standard transformer encoder-decoder architecture with self-attention. In the experiment group, we apply SSE-SE towards both encoder and decoder by replacing corresponding vocabularies' embeddings in the source and target sentences. We trained on the standard WMT 2014 English to German dataset which consists of roughly 4.5 million parallel sentence pairs and tested on WMT 2008 to 2018 news-test sets. Sentences were encoded into 32,000 tokens using a byte-pair encoding. We use the SentencePiece, OpenNMT and SacreBLEU implementations in our experiments. We trained the 6-layer transformer base model on a single machine with 4 NVIDIA V100 GPUs for 20,000 steps. We use the same dropout rate of 0.1 and label smoothing value of 0.1 for the baseline model and our SSE-enhanced model. Both models have dimensionality of embeddings as $d = 512$. When decoding, we use beam search with the beam size of 4 and length penalty of 0.6 and replace unknown words using attention. For both models, we average last 5 checkpoints (we save checkpoints every 10,000 steps) and evaluate the model's performances on the test datasets using BLEU scores. The only difference between the two models is whether or not we use our proposed SSE-SE with $p = 0.01$ in Equation~\ref{eq:sse-se} for both encoder and decoder embedding layers.

\subsection{BERT}
In the first leg of experiments, we crawled one million user reviews data from IMDB and pre-trained the BERT-Base model (12 blocks) for $500,000$ steps using sequences of maximum length 512 and batch size of 8, learning rates of $2e^{-5}$ for both models using one NVIDIA V100 GPU. Then we pre-trained on a mixture of our crawled reviews and reviews in IMDB sentiment classification tasks (250K reviews in train and 250K reviews in test) for another $200,000$ steps before training for another $100,000$ steps for the reviews in IMDB sentiment classification task only. In total, both models are pre-trained on the same datasets for $800,000$ steps with the only difference being our model utilizes SSE-SE. In the second leg of experiments, we fine-tuned the two models obtained in the first-leg experiments on two sentiment classification tasks: IMDB sentiment classification task and SST-2 sentiment classification task. The goal of pre-training on IMDB dataset but fine-tuning for SST-2 task is to explore whether SSE-SE can play a role in transfer learning.

The results are summarized in Table~\ref{tb:bert-imdb} for IMDB sentiment task. In experiments, we use maximum sequence length of 512, learning rate of $2e^{-5}$, dropout probability of $0.1$ and we run fine-tuning for 1 epoch for the two pre-trained models we obtained before. For the Google pre-trained BERT-base model, we find that we need to run a minimum of 2 epochs. This shows that pre-training can speed up the fine-tuning. We find that Google pre-trained model performs worst in accuracy because it was only pre-trained on Wikipedia and books corpus while ours have seen many additional user reviews. We also find that SSE-SE pre-trained model can achieve accuracy of 0.9542 after fine-tuning for one epoch only. On the contrast, the accuracy is only 0.9518 without SSE-SE for embeddings associated with output $y_i$. 

For the SST-2 task, we use maximum sequence length of 128, learning rate of $2e^{-5}$, dropout probability of 0.1 and we run fine-tuning for 3 epochs for 
all 3 models in Table~\ref{tb:bert-sst2}. We report AUC, accuracy and F1 score for dev data. For test results, we submitted our predictions to Glue website for the official evaluation. We find that even in transfer learning, our SSE-SE pre-trained model still enjoys advantages over Google pre-trained model and our pre-trained model without SSE-SE. Our SSE-SE pre-trained model achieves 94.3\% accuracy on SST-2 test set versus 93.6 and 93.8 respectively. If we are using SSE-SE for both pre-training and fine-tuning, we can achieve 94.5\% accuracy on the SST-2 test set, which approaches the 94.9 score reported by the BERT-Large model. SSE probability of 0.01 is used for fine-tuning.

% \begin{table}
%   \caption{Our proposed SSE-SE helps the ensemble model perform better during Kaggle Jigsaw Unintended Bias in Toxicity Classification Competition. We report the public leaderboard scores. Ensemble model consists 4 single models initialized from the different pre-training checkpoints. }
%   \label{tb:ensemble}
%   \centering
%   \resizebox{0.8\columnwidth}{!}{
%   \begin{tabular}{ccc}
%     \toprule
%     & \multicolumn{2}{c}{Bias AUC}                   \\
%     \cmidrule(r){2-3} 
%     Model     & Single Model & Ensemble Model   \\
%     \midrule
%     Multi-task De-bias Bert            & 0.94132 &      0.94192  \\
%     SSE-SE + Multi-task De-bias Bert   & \bfseries{0.94139} & \bfseries{0.94216} \\
%     \bottomrule
%   \end{tabular}
%   }
% \end{table}

%\section{Appendix}
\subsection{Proofs}

\begin{figure}
  \centering
  \includegraphics[width=0.835\linewidth]{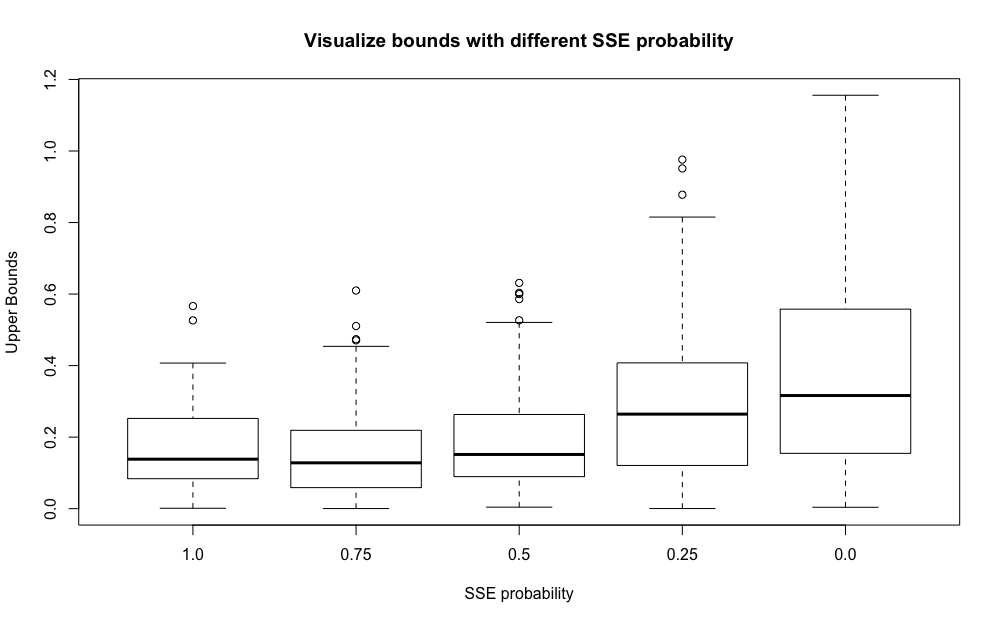}
  \caption{Simulation of a bound on $\rho_{L,n}$ for the movielens1M dataset.  Throughout the simulation, $L$ is replaced with $\ell$ (which will bound $\rho_{L,n}$ by Jensen's inequality).  The SSE probability parameter dictates the probability of transitioning.  When this is $0$ (box plot on the right), the distribution is that of the samples from the standard Rademacher complexity (without the sup and expectation).  As we increase the transition probability, the values for $\rho_{L,n}$ get smaller.
%   The actual transition probabilities to be used is determined by the knowledge graphs and specific datasets and requires only minimal tuning.
  }\vspace{-12pt}
  \label{fig:sim}
\end{figure}

Throughout this section, we will suppress the probability parameters, $p(.,.|\Phi) = p(.,.)$.

\begin{proof}[Proof of Theorem \ref{thm:rademacher}]
Consider the following variability term,
\begin{equation}
   \sup_\Theta | S(\Theta) - S_n(\Theta) |. 
\end{equation}
Let us break the variability term into two components
\[
\mathbb{E}_{X,Y} \sup_\Theta \left| S_n(\Theta) - \mathbb E_{Y|X} [S_n(\Theta)] \right| + \mathbb{E}_{X,Y} \sup_\Theta \left| \mathbb E_{Y|X} [S_n(\Theta)] - S(\Theta) \right|,
\]
where $X,Y$ represent the random input and label.
To control the first term, we introduce a ghost dataset $(x_{i}, y_{i}')$, where $y_{i}'$ are independently and identically distributed according to $y_i | x_i$.
Define
\begin{equation}
    S'_n(\Theta) = \sum_{i} \sum_{\bo k} p(\bo j^i, \bo k) \ell(\bo E[\bo k],y_i'|\Theta)
\end{equation} be the empirical SSE risk with respect to this ghost dataset.
 
 We will rewrite $\mathbb E_{Y|X} [S_n(\Theta)]$ in terms of the ghost dataset and apply Jensen's inequality and law of iterated conditional expectation:
\begin{align}
    &\mathbb{E} \sup_\Theta \left| \mathbb E_{Y|X} [S_n(\Theta)] -  S_n(\Theta) \right| \\
    &= \mathbb{E} \sup_\Theta \left| \mathbb{E}_{Y' | X} \left[S'_n(\Theta) - S_n(\Theta) \right] \right| \\
    &\leq \mathbb{E}\mathbb{E}_{Y' | X} \left[ \sup_\Theta \left| S'_n(\Theta) - S_n(\Theta) \right| \right] \\
    &= \mathbb{E} \sup_\Theta  \left| S'_n(\Theta) - S_n(\Theta) \right|.
\end{align}

Notice that 
\begin{align*}
S'_n(\Theta) - S_n(\Theta) = \sum_{i} \sum_{\bo k} p(\bo j^i, \bo k) \left( \ell(\bo E[\bo k],y_i'|\Theta) - \ell(\bo E[\bo k],y_i|\Theta)\right)\\
= \sum_{i} \sum_{\bo k} p(\bo j^i, \bo k) \left( e(\bo E[\bo k],y_i'|\Theta) - e(\bo E[\bo k],y_i|\Theta)\right).
\end{align*}
Because $y_i, y_i' | X$ are independent the term $(\sum_{\bo k} p(\bo j^i, \bo k) \left( e(\bo E[\bo k],y_i'|\Theta) - e(\bo E[\bo k],y_i|\Theta)\right))_i$ is a vector of symmetric independent random variables.
Thus its distribution is not effected by multiplication by arbitrary Rademacher vectors $\sigma_i \in \{-1,+1\}$.
\[
\mathbb{E} \sup_\Theta  \left| S'_n(\Theta) - S_n(\Theta) \right| = \mathbb{E} \sup_\Theta \left| \sum_{i} \sigma_i \sum_{\bo k} p(\bo j^i, \bo k) \left( e(\bo E[\bo k],y_i'|\Theta) - e(\bo E[\bo k],y_i|\Theta)\right) \right|.
\]
But this is bounded by 
\[
2 \mathbb{E} \mathbb E_\sigma \sup_\Theta \left| \sum_{i} \sigma_i \sum_{\bo k} p(\bo j^i, \bo k) e(\bo E[\bo k],y_i|\Theta) \right|.
\]

For the second term,
\[
\mathbb{E} \sup_\Theta \left| \mathbb E_{Y|X} [S_n(\Theta)] - S(\Theta) \right|
\]
we will introduce a second ghost dataset $x_i',y_i'$ drawn iid to $x_i,y_i$.
Because we are augmenting the input then this results in a new ghost encoding $\bo j'^{i}$.
Let
\begin{equation}
    S'_n(\Theta) = \sum_{i} \sum_{\bo k} p(\bo j'^{i}, \bo k) \ell(\bo E[\bo k],y_i'|\Theta)
\end{equation} be the empirical risk with respect to this ghost dataset.
Then we have that 
\[
S(\Theta) = \mathbb E_{X'} \mathbb E_{Y' | X'} S_n'(\Theta) 
\]
Thus,
\begin{align}
    &\mathbb{E} \sup_\Theta \left| \mathbb E_{Y|X} [S_n(\Theta)] -  S(\Theta) \right| \\
    &= \mathbb{E} \sup_\Theta \left| \mathbb{E}_{X'} \left[E_{Y|X} [S_n(\Theta)] - E_{Y'|X'} [S'_n(\Theta)] \right] \right| \\
    &\leq \mathbb{E}\mathbb{E}_{X'} \left[ \sup_\Theta \left| E_{Y|X} [S_n(\Theta)] - E_{Y'|X'} [S'_n(\Theta)] \right| \right] \\
    &= \mathbb{E} \sup_\Theta \left| E_{Y|X} [S_n(\Theta)] - E_{Y'|X'} [S'_n(\Theta)] \right|.
\end{align}
Notice that we may write,
\[
E_{Y|X} [S_n(\Theta)] - E_{Y'|X'} [S'_n(\Theta)] = \sum_{i} \sum_{\bo k} \left( p(\bo j^{i}, \bo k) - p(\bo j'^{i}, \bo k) \right) L(\bo E[\bo k]|\Theta)
\]
Again we may introduce a second set of Rademacher random variables $\sigma'_i$, which results in 
\[
\mathbb{E} \sup_\Theta \left| E_{Y|X} [S_n(\Theta)] - E_{Y'|X'} [S'_n(\Theta)] \right| \le 2 \mathbb E \mathbb E_{\sigma'} \sup_\Theta \left| \sum_{i} \sigma'_i \sum_{\bo k} p(\bo j^{i}, \bo k) L(\bo E[\bo k]|\Theta) \right|.
\]
And this is bounded by 
\[
2 \mathbb E \mathbb E_{\sigma'} \sup_\Theta \left| \sum_{i} \sigma'_i \sum_{\bo k} p(\bo j^{i}, \bo k) L(\bo E[\bo k]|\Theta) \right| \le 2 \mathbb E \sup_\Theta \left| \sum_{i} \sigma'_i \sum_{\bo k} p(\bo j^{i}, \bo k) \ell(\bo E[\bo k],y_i|\Theta) \right|
\]
by Jensen's inequality again.
\end{proof}

\begin{proof}[Proof of Theorem \ref{thm:main}]
It is clear that $2 \mathcal B \ge B(\hat \Theta) + B(\Theta^*)$.
It remains to show our concentration inequality.
Consider changing a single sample, $(x_i,y_i)$ to $(x_i',y_i')$, thus resulting in the SSE empirical risk, $S_{n,i}(\Theta)$.
Thus,
\begin{align*}
&S_n(\Theta) - S_{n,i}(\Theta) = \sum_{\bo k} p(\bo j^i, \bo k) \cdot \ell(\bo E[\bo k],y_i|\Theta) - \sum_{\bo k} p(\bo j'^i, \bo k) \cdot \ell(\bo E[\bo k],y_i'|\Theta)\\
&= \sum_{\bo k} p(\bo j^i, \bo k) \cdot \left( \ell(\bo E[\bo k],y_i|\Theta) - \ell(\bo E[\bo k],y_i'|\Theta)\right) + \sum_{\bo k} \left( p(\bo j'^i, \bo k) - p(\bo j^i, \bo k) \right) \cdot \ell(\bo E[\bo k],y_i'|\Theta)\\
&\le b \left( \sum_{\bo k} p(\bo j^i, \bo k) + \sum_{\bo k} p(\bo j'^i, \bo k) \right) \le 2b.
\end{align*}
Then the result follows from McDiarmid's inequality.

\end{proof}

\end{document}